\documentclass[10pt]{article}
\usepackage[numbers]{natbib}
\usepackage{soul}
\usepackage{url}
\usepackage[utf8]{inputenc}
\usepackage[small]{caption}
\usepackage{graphicx}
\usepackage{hyperref}       
\usepackage{booktabs}       
\usepackage{nicefrac}       
\usepackage{microtype}      
\usepackage{color}
\usepackage{times}
\usepackage{graphicx} 
\usepackage{subcaption}
\usepackage{algorithm}
\usepackage{algorithmic}
\usepackage{paralist}
\usepackage{multirow}
\usepackage{diagbox}
\usepackage{amsmath,amssymb,amsthm,mathrsfs,amsfonts,dsfont}
\usepackage{thmtools,thm-restate}
\usepackage{geometry}
\newcommand\ind[1]{\ensuremath{\mathds{1}\left[#1\right]}}

\def\argmin{\ensuremath{\mbox{argmin}}}

\newcommand{\D}{\mathcal{D}}

\newcommand{\B}{\mathcal{B}}
\newcommand{\ES}{\mathcal{S}}
\newcommand{\M}{\mathcal{M}}
\newcommand{\E}[2]{\mathbf{E}_{#1}\left[#2\right]}

\renewcommand{\P}[2]{\mathbf{P}_{#1}\left(#2\right)}
\newcommand{\R}{\mathbb{R}}

\newcommand{\X}{\mathcal{X}}

\newcommand{\Y}{\mathcal{Y}}
\newcommand{\Z}{\mathcal{Z}}

\newcommand{\thmref}[1]{Theorem~\ref{#1}}
\newcommand{\lemref}[1]{Lemma~\ref{#1}}

\newtheorem{definition}{Definition}
\newtheorem{example}{Example}

\title{Data Poisoning against Differentially-Private Learners: \\ Attacks and Defenses}

\author{Yuzhe Ma \\ ma234@wisc.edu   \and Xiaojin Zhu\\ jerryzhu@cs.wisc.edu \and Justin Hsu \\justhsu@cs.wisc.edu \and Department of Computer Sciences, University of Wisconsin--Madison}

\date{}

\begin{document}
	
	\maketitle
	
	\begin{abstract}
\emph{Data poisoning} attacks aim to manipulate the model produced by a learning algorithm by adversarially modifying the training set. 
We consider \emph{differential privacy} as a defensive measure against this type of attack.
We show that such learners are resistant to data poisoning attacks when the adversary is only able to poison a small number of items.
However, this protection degrades as the adversary poisons more data.
To illustrate, we design attack algorithms targeting \emph{objective} and \emph{output perturbation} learners, two standard approaches to differentially-private machine learning.
Experiments show that our methods are effective when the attacker is allowed to poison sufficiently many training items.
\end{abstract}
	
	\section{Introduction}

As machine learning is increasingly used for consequential decisions in the real
world, the security concerns have received more and more scrutiny. Most machine
learning systems were originally designed without much thought to security, but
researchers have since identified several kinds of attacks under the umbrella of
\emph{adversarial machine learning}. The most well-known example is
\emph{adversarial examples}~\cite{szegedy2013intriguing}, where inputs are
crafted to fool classifiers. More subtle methods aim to recover training
examples \cite{fredrikson2015model}, or even extract the model itself
\cite{tramer2016stealing}.

In \emph{data poisoning}
attacks~\cite{biggio2012poisoning,koh2018stronger,Mei2015Machine,alfeld2016data,li2016data,xiao2015feature,munoz2017towards,jun2018adversarial,zhao2017efficient,zhang2019online},
the adversary corrupts examples at training time to manipulate the learned
model.
As is generally the case in adversarial machine learning, the rise of data
poisoning attacks has outpaced the development of robust defenses. While
attacks are often evaluated against specific target learning procedures,
effective defenses should provide protection---ideally guaranteed---against a
broad class of attacks. In this paper, we study \emph{differential
privacy}~\cite{dwork2006calibrating,dwork2011differential} as a general
defensive technique against data poisoning. 
While differential privacy was originally designed to formalize privacy---the
output should not depend too much on any single individual's private data---it
also provides a defense against data poisoning. Concretely, we establish
quantitative bounds on how much an adversary can change the distribution over
learned models by manipulating a fixed number of training items. We are not the
first to design defenses to data poisoning~\cite{steinhardt2017certified,raghunathan2018certified},
but our proposal has several notable strengths compared to previous work.
First, our defense provides provable protection against a worst-case adversary.
Second, our defense is general. Differentially-private learning procedures are
known for a wide variety of tasks, and new algorithms are continually being
proposed. These learners are all resilient against data poisoning.

We complement our theoretical bounds with an empirical evaluation of data
poisoning attacks on differentially private learners. Concretely, we design an
attack based on stochastic gradient descent to search for effective training
examples to corrupt against specific learning algorithms. We evaluate attack on
two private learners: objective perturbation~\cite{kifer2012private} and output
perturbation~\cite{chaudhuri2011differentially}. Our evaluation confirms the
theoretical prediction that private learners are vulnerable to data poisoning
attacks when the adversary can poison sufficiently many examples.  A gap remains
between performance of attack and theoretical limit of data poisoning
attacks---it seems like differential privacy provides a stronger defense than
predicted by theory. We discuss possible causes and leave further investigation
to future work.

	\section{Preliminaries}\label{threat-model}

\paragraph*{Differential Privacy.}
Let the data space be $\Z =\X\times\Y$, where $\X$ is the feature space and $\Y$
is the label space. We write $\D=\cup_{n=0}^\infty \Z^n$ for the space of all
training sets, and write $D\in\D$ for a particular training set. A randomized
learner $\M :\D \times \R^d \mapsto \Theta$ maps a training set $D$ and noise
$b\in\R^d$ drawn from a distribution $\nu$ to a model in the model space
$\Theta$.
\begin{definition}{(Differential Privacy)}
  $\M$ is $(\epsilon,\delta)$-\emph{differentially-private} if $\forall  D, \tilde D\in\D$ that differ by one item, and $\forall \ES\subset \Theta$,
\begin{equation}
\P{}{\M(D,b)\in \ES}\le e^\epsilon\P{}{\M(\tilde D,b)\in \ES}+\delta,
\end{equation}
where the probability is taken over $b \sim \nu$.  When $\delta=0$, we say that
$\M$ is $\epsilon$-\emph{differentially private}.
\end{definition}
Many standard machine learning tasks can be phrased as optimization problems
modeling empirical risk minimization (ERM). Broadly speaking, there are two
families of techniques in the privacy literature for solving these problems.
Objective perturbation injects noise into the objective function of a vanilla
machine learner to train a randomized
model~\cite{chaudhuri2011differentially,kifer2012private,nozari2016differentially}.
In contrast, output perturbation runs the vanilla learner but after training, it
randomly perturbs the output model~\cite{chaudhuri2011differentially}. We will
consider data poisoning attacks against both kinds of learners.

\paragraph*{Threat Model.}
We first fix the adversarial attack setting.

\emph{Knowledge of the attacker:} 
The attacker has full knowledge of the full training set $D$ and the
differentially-private machine learner $\M$, including the noise distribution
$\nu$. However, the attacker does not know the realization of $b$.

\emph{Ability of the attacker:}
The attacker is able to modify the clean training set $D$ by changing the
features or labels, adding additional items, or deleting existing items.
Nevertheless, we limit the attacker to modify at most $k$ items. Formally, the
poisoned training set $\tilde D$ must lie in the ball $\B(D,k)\subset \D$, where
the center is the clean data $D$ and the radius $k$ is the maximum number of item
changes.

\emph{Goal of the attacker:}
the attacker aims to force the (stochastic) model learned from the poisoned data
$\tilde\theta=\M(\tilde D, b)$ to achieve certain nefarious target. Formally, we define a cost function
$C:\Theta\mapsto \R$ that measures how much the poisoned model $\tilde\theta$
deviates from the attack target.
Then the attack problem can be formulated as minimizing an objective function $J$:
\begin{equation}\label{attack:opt}
\min_{ \tilde D \in \B(D,k)} \quad J(\tilde  D):=\E{b}{ C(\M( \tilde D, b))}.
\end{equation}
This formulation differs from previous works
(e.g.~\cite{biggio2012poisoning,Mei2015Machine,xiao2015feature}) because the
differentially-private learner produces a stochastic model, so the objective
takes the expected value of $C$. 
 We now show a few attack problem instances formulated as~\eqref{attack:opt}.

\begin{example}{(Parameter-Targeting Attack)}
\label{ex:parameter}
If the attacker wants the output model to be close to a target model
$\theta^\prime$, the attacker can define $C(\tilde \theta)=\frac{1}{2}\|\tilde
\theta-\theta^\prime\|^2$, a non-negative cost. Minimizing $J( \tilde D)$
pushes the poisoned model close to $\theta^\prime$ in expectation.
\end{example}

\begin{example}{(Label-Targeting Attack)}
\label{ex:targeting}
If the attacker wants small prediction error on an evaluation set
$\{z_i^*\}_{i\in[m]}$, the attacker can define $C(\tilde
\theta)=\frac{1}{m}\sum_{i=1}^m{\ell(\tilde \theta, z_i^*)}$ where $\ell(\tilde
\theta, z_i^*)$ is the loss on item $z_i^*=(x_i^*, y_i^*)$. Minimizing
$J( \tilde D)$ pushes the prediction on $x_i^*$ towards the target label $y_i^*$.
\end{example}

\begin{example}{(Label-Aversion Attack)}
\label{ex:aversion}
If the attacker wants to induce large prediction error on an evaluation set
$\{z_i^*\}_{i\in[m]}$, the attacker can define $C(\tilde
\theta)=-\frac{1}{m}\sum_{i=1}^m\ell(\tilde \theta, z_i^*)$, a non-positive
cost.  Minimizing $J( \tilde D)$ pushes the prediction on $x_i^*$ away from the
true label $y_i^*$.
\end{example}

Now, we turn to our two main contributions. $i)$ We show that
differentially-private learners have \emph{some} natural immunity against data
poisoning attacks.  Specifically, in section~\ref{sec:immunity} we present a
lower bound on how much the attacker can reduce $J( \tilde D)$ under a fixed $k$
(the number of changes). $ii)$ When $k$ increases, the attacker's power grows
and it becomes possible to attack differentially-private learners. In
section~\ref{sec:attack} we propose a stochastic gradient descent algorithm to
solve the attack optimization problem~\eqref{attack:opt}. To our knowledge, ours
is the first data poisoning attack targeting differentially-private learners.

	\section{Differential Privacy Resists Data Poisoning} 
\label{sec:immunity}

Fixing the number of poisoned items $k$ and a differentially private learner,
can a powerful attacker achieve arbitrarily low attack cost $J( \tilde D)$?  We
show that the answer is no: differential privacy implies a lower bound on how
far $J( \tilde D)$ can be reduced, so private learners are inherently resistant
to data poisoning attacks.  Previous work \cite{lecuyer2018certified} proposed
defenses based on differential privacy, but for test-time attacks. While
effective, these defenses require the classifier to make randomized predictions.
In contrast, our defense against data poisoning only requires the \emph{learning
algorithm}, not the classification process, to be randomized. 

Our first result is for $\epsilon$-differentially private learners:
if an attacker can manipulate at most $k$ items in the clean data, then the
attack cost is lower bounded. 

\begin{restatable}{theorem}{thmDPdefenceEps}
\label{thm:DP-defence-eps}
Let $\M$ be an $\epsilon$-differentially-private learner. Let $J( \tilde D)$ be the attack cost, where $\tilde D\in \B(D,k)$, then 
\begin{equation}\label{eq:lowerbound-eps}
\begin{aligned}
&J( \tilde D)\ge e^{-k\epsilon}J(D) && (C\ge 0),\\
&J( \tilde D)\ge e^{k\epsilon}J(D) &&(C\le 0) .
\end{aligned}
\end{equation}
\end{restatable}
\begin{proof}
For any poisoned data $\tilde D\in \B(D,k)$, one can view it as being produced by manipulating the clean data $D$ $k$ times, where each time one item is changed.  
Thus one can apply~\lemref{lem:DP-defence-eps} $k$ times to obtain~\eqref{eq:lowerbound-eps}.
\end{proof}

\begin{restatable}{lemma}{lemLBeps}
\label{lem:DP-defence-eps}
Let $\M$ be an $\epsilon$-differentially-private learner. Let $J( \tilde D)$ be the attack cost, where $\tilde D\in \B(D,1)$, then 
\begin{equation}
\begin{aligned}
&J( \tilde D)\ge e^{-\epsilon}J( D), && (C\ge 0),\\
&J( \tilde D)\ge e^{\epsilon}J( D), && (C\le 0).
\end{aligned}
\end{equation}
\end{restatable}
\begin{proof}
We first consider $C\ge 0$. Define $\Theta(a)=\{\theta:  C(\theta)> a\}$, $\forall a\ge 0$. Since $\tilde D\in \B(D, 1)$, $\tilde D$ and $D$ differ by at most one item, thus differential privacy guarantees that $\forall a\ge0$,
\begin{equation}
\P{}{\M(D,b)\in \Theta(a)}\le e^{\epsilon}\P{}{\M(\tilde D, b)\in \Theta(a)}.
\end{equation}
Since $J(D)$ is non-negative, by the integral identity, we have
\begin{equation}
\begin{aligned}
J( D)&=\E{b}{C(\M(D,b))}= \int_{0}^\infty\P{}{C(\M(D,b))> a}da=\int_{0}^\infty\P{}{\M(D,b)\in \Theta(a)}da\\
&\le  \int_{0}^\infty e^{\epsilon}\P{}{\M(\tilde D, b)\in \Theta(a)}da=e^{\epsilon}\int_0^\infty \P{}{C(\M(\tilde D,b))> a}da=e^{\epsilon}J( \tilde D).
\end{aligned}
\end{equation}
Thus $J( \tilde D)\ge e^{-\epsilon}J( D)$.
Next we consider $C\le 0$. Define $\Theta(a)=\{\theta:  C(\theta)< a\}$, $\forall a\le 0$. Again due to $\M$ being differentially private,
\begin{equation}
\P{}{\M(\tilde D,b)\in \Theta(a)}\le e^{\epsilon}\P{}{\M(D, b)\in \Theta(a)}.
\end{equation}
Since $J(\tilde D)$ is non-positive, by the integral identity, we have
\begin{equation}
\begin{aligned}
J( \tilde D)&=\E{b}{C(\M(\tilde D,b))}= -\int_{-\infty}^{0}\P{}{C(\M(\tilde D,b))< a}da=-\int_{-\infty}^0\P{}{\M(\tilde D,b)\in \Theta(a)}da\\
&\ge  -\int_{-\infty}^0 e^{\epsilon}\P{}{\M(D, b)\in \Theta(a)}da=-\int_{-\infty}^0e^{\epsilon} \P{}{C(\M(D,b))< a}da=e^{\epsilon}J(D).
\end{aligned}
\end{equation}
\end{proof}

Note that for $C\ge 0$,
\thmref{thm:DP-defence-eps} shows that no attacker can achieve the trivial lower
bound $J(\tilde D)=0$. For $C\le 0$, although $J(\tilde D)$ could be unbounded
from below,~\thmref{thm:DP-defence-eps} shows that no attacker can reduce
$J(\tilde D)$ arbitrarily. These results crucially rely on the differential
privacy property of $\M$.  Corollary~\ref{col:num-of-items-eps} re-states the theorem in
terms of the minimum number of items an attacker has to modify in order to sufficiently
reduce the attack cost $J( \tilde D)$.

\begin{restatable}{corollary}{corNumEps}
\label{col:num-of-items-eps}
Let $\M$ be an $\epsilon$-differentially-private learner. Assume $J(D)\neq 0$.
Let $\tau\ge 1$. Then the attacker has to modify at least $k \ge
\lceil\frac{1}{\epsilon}\log \tau\rceil$ items in $D$ in order to achieve
\begin{enumerate}[i).]
\centering
\item $J( \tilde D)\le {1 \over \tau} J(D) \quad (C\ge 0)$.
\item $J( \tilde D)\le \tau J(D) \quad (C\le 0)$.
\end{enumerate}
\end{restatable}
\begin{proof}
For $C\ge 0$, by~\thmref{thm:DP-defence-eps}, if  $ \tilde D\in \B(D,k)$ then $J( \tilde D)\ge e^{-k\epsilon}J( D)$. Since we require $J( \tilde D)\le {1 \over \tau} J(D)$, we have $e^{-k\epsilon}J( D)\le {1 \over \tau} J(D)$, thus $k\ge \lceil\frac{1}{\epsilon}\log {\tau}\rceil$. Similarly, one can prove the case for $C\le 0$.
\end{proof}

To generalize~\thmref{thm:DP-defence-eps} to
$(\epsilon,\delta)$-private learners, we need an additional
assumption that $C$ is bounded.

\begin{restatable}{theorem}{thmDPdefenceEpsDelta}
\label{thm:DP-defence-epsdelta}
Let $\M$ be $(\epsilon,\delta)$-differentially-private and $J(\tilde D)$
be the attack cost, where $|C|\le \bar C$ and $\tilde D\in \B(D,k)$. Then,
\begin{equation}
\begin{aligned}
& J( \tilde D)\ge \max\{e^{-k\epsilon}(J( D)+\frac{\bar
C\delta}{e^{\epsilon}-1})-\frac{\bar C\delta}{e^{\epsilon}-1},0\} \quad (C\ge 0),\\
& J( \tilde D)\ge \max\{e^{k\epsilon}(J( D)-\frac{\bar
C\delta}{e^{\epsilon}-1})+\frac{\bar C\delta}{e^{\epsilon}-1},-\bar C\} \quad (C\le 0).
\end{aligned}
\end{equation}
\end{restatable}

\begin{proof}
We first consider $0\le C\le \bar C$. For any poisoned data $\tilde D\in \B(D,k)$, one can view it as being produced by manipulating the clean data $D$ $k$ times, where each time one item is changed. Thus we have a sequence of intermediate $\tilde D_i$ $(0\le i\le k)$, the data set where $i$ items are poisoned. We define $\tilde D_0=D$ and $\tilde D_k=\tilde D$. Let $J_i=J(\tilde D_i)$, then by~\lemref{lem:DP-defence-eps-delta}
\begin{equation}
J_i\ge e^{-\epsilon} (J_{i-1}-\bar C\delta), \forall 1\le i\le k.
\end{equation}
Adding $\frac{\bar C\delta}{e^{\epsilon}-1}$ to both sides, we have
\begin{equation}\label{recursive}
J_i+\frac{\bar C\delta}{e^{\epsilon}-1}\ge e^{-\epsilon} (J_{i-1}-\bar C\delta)+\frac{\bar C\delta}{e^{\epsilon}-1}=e^{-\epsilon}(J_{i-1}+\frac{\bar C\delta}{e^{\epsilon}-1}).
\end{equation}
We then recursively apply~\eqref{recursive} $k$ times to obtain
\begin{equation}
J_k+\frac{\bar C\delta}{e^{\epsilon}-1}\ge e^{-k\epsilon}(J_0+\frac{\bar C\delta}{e^{\epsilon}-1})=e^{-k\epsilon}(J(D)+\frac{\bar C\delta}{e^{\epsilon}-1}).
\end{equation}
Note $J_k=J(\tilde D_k)=J(\tilde D)$, thus we have
\begin{equation}
J( \tilde D)\ge e^{-k\epsilon}(J( D)+\frac{\bar C\delta}{e^{\epsilon}-1})-\frac{\bar C\delta}{e^{\epsilon}-1}.
\end{equation}
Also note that $J( \tilde D)\ge 0$ trivially holds due to $C\ge 0$, thus
\begin{equation}
J( \tilde D)\ge \max\{e^{-k\epsilon}(J( D)+\frac{\bar C\delta}{e^{\epsilon}-1})-\frac{\bar C\delta}{e^{\epsilon}-1},0\}.
\end{equation}
Next, we consider $-\bar C\le C\le 0$. Define $J_i$ as before, then by~\lemref{lem:DP-defence-eps-delta}
\begin{equation}
J_i\ge e^{\epsilon} J_{i-1}-\bar C\delta, \forall 1\le i\le k.
\end{equation}
Subtracting both sides by $\frac{\bar C\delta}{e^\epsilon-1}$ we have
\begin{equation}
J_i-\frac{\bar C\delta}{e^\epsilon-1}\ge e^{-\epsilon} (J_{i-1}-\frac{\bar C\delta}{e^\epsilon-1}), \forall 1\le i\le k.
\end{equation}
Thus we have
\begin{equation}
J(\tilde D)-\frac{\bar C\delta}{e^\epsilon-1}\ge e^{-k\epsilon}(J(D)-\frac{\bar C\delta}{e^\epsilon-1}).
\end{equation}
Combined with the trivial lower bound $-\bar C$, we have
\begin{equation}
J(\tilde D)\ge \max\{e^{-k\epsilon}(J(D)-\frac{\bar C\delta}{e^\epsilon-1})+\frac{\bar C\delta}{e^\epsilon-1},-\bar C\}.
\end{equation}
\end{proof}

\begin{restatable}{lemma}{lemLBEpsDelta}
\label{lem:DP-defence-eps-delta}
Let $\M$ be an $(\epsilon,\delta)$-differentially-private learner. Let $J(\tilde D)$ be the attack cost, where $|C|\le \bar C$ and $\tilde D\in \B(D,1)$, then 
\begin{equation}
\begin{aligned}
& J( \tilde D)\ge e^{-\epsilon} (J( D)-\bar C\delta), && (C\ge 0),\\
& J( \tilde D)\ge e^{\epsilon} J( D)-\bar C\delta, && (C\le 0).
\end{aligned}
\end{equation}
\end{restatable}
\begin{proof}
We first consider $0\le C\le \bar C$. Define $\Theta(a)=\{\theta:  C(\theta)> a\}$, $\forall a\ge 0$. Since $\tilde D\in \B(D, 1)$, $\tilde D$ and $D$ differ by at most one item, thus the differential privacy guarantees that $\forall a\ge0$,
\begin{equation}
\P{}{\M(D,b)\in \Theta(a)}\le e^{\epsilon}\P{}{\M(\tilde D, b)\in \Theta(a)}+\delta.
\end{equation}
Note that $J(D)$ is non-negative, thus by the integral identity, we have
\begin{equation}
\begin{aligned}
J(D)&=\E{b}{C(\M(D,b))}= \int_{0}^{\bar C}\P{}{C(\M(D,b))> a}da=\int_{0}^{\bar C}\P{}{\M(D,b)\in \Theta(a)}da\\
&\le  \int_{0}^{\bar C} \left(e^{\epsilon}\P{}{\M(\tilde D, b)\in \Theta(a)}+\delta\right) da=e^{\epsilon}\int_0^{\bar C} \P{}{C(\M(\tilde D,b))> a}da+\bar C\delta=e^{\epsilon}J( \tilde D)+\bar C\delta.
\end{aligned}
\end{equation}
Rearranging we have $J( \tilde D)\ge e^{-\epsilon} (J( D)-\bar C\delta)$.

Next we consider $-\bar C\le C\le 0$. Define $\Theta(a)=\{\theta:  C(\theta)< a\}$, $\forall a\le 0$ Again, due to $\M$ being differentially private, we have
\begin{equation}
\P{}{\M(\tilde D,b)\in \Theta(a)}\le e^{\epsilon}\P{}{\M(D, b)\in \Theta(a)}+\delta.
\end{equation}
Since $J(\tilde D)$ is non-positive, by the integral identity, we have
\begin{equation}
\begin{aligned}
J(\tilde D)&=\E{b}{C(\M(\tilde D,b))}= -\int_{-\bar C}^0\P{}{C(\M(\tilde D,b))< a}da\\
&=-\int_{-\bar C}^0\P{}{\M(\tilde D,b)\in \Theta(a)}da\ge  -\int_{-\bar C}^0 \left(e^{\epsilon}\P{}{\M(D, b)\in \Theta(a)}+\delta \right) da\\
&=-e^{\epsilon}\int_{-\bar C}^0 \P{}{C(\M(D,b))< a}da-\bar C\delta=e^{\epsilon}J(D)-\bar C\delta.
\end{aligned}
\end{equation}
\end{proof}


As a corollary, we can lower-bound the minimum number of
modifications needed to sufficiently reduce attack cost.
\begin{restatable}{corollary}{corNumEpsDelta}
\label{col:num-of-items-epsdelta}
Let $\M$ be an $(\epsilon,\delta)$-differentially-private learner. 
Assume $J( D)\neq 0$.
Then
\begin{enumerate}[i).]
\item
$(0\le C\le \bar C)$: in order to achieve $J( \tilde D)\le {1 \over \tau} J(D)$ for $\tau\ge 1$,
the attacker has to modify at least the following number of items in $D$.
\begin{equation}\label{eq:num-of-items-eps-delta}
k \ge \lceil \frac{1}{\epsilon}\log \frac{(e^{\epsilon}-1) J(D)\tau+\bar C\delta\tau}{(e^{\epsilon}-1) J(D)+\bar C\delta\tau} \rceil.
\end{equation}
\item
$(-\bar C \le C\le 0)$: in order to achieve $J( \tilde D)\le \tau J(D)$ for $\tau\in [1, -\frac{\bar C}{J(D)}]$,
the attacker has to modify at least the following number of items in $D$.
\begin{equation}\label{eq:num-of-items-eps-delta-negative}
k \ge \lceil \frac{1}{\epsilon}\log \frac{(e^{\epsilon}-1) J(D)\tau-\bar C\delta}{(e^{\epsilon}-1) J(D)-\bar C\delta} \rceil.
\end{equation}
\end{enumerate}
\end{restatable}
\begin{proof}
For $0\le C\le \bar C$, by~\thmref{thm:DP-defence-epsdelta}, if $ \tilde D\in \B(D,k)$, then $J( \tilde D)\ge \max\{e^{-k\epsilon}(J( D)+\frac{\bar C \delta}{e^{\epsilon}-1})-\frac{\bar C \delta}{e^{\epsilon}-1},0\}$. Since we require $J( \tilde D)\le \frac{1}{\tau}J(D)$, we have $e^{-k\epsilon}(J( D)+\frac{\bar C \delta}{e^{\epsilon}-1})-\frac{\bar C \delta}{e^{\epsilon}-1}\le\frac{1}{\tau}J(D)$. Rearranging gives the result. The proof for $-\bar C\le C\le 0$ is similar.
\end{proof}

Note that in~\eqref{eq:num-of-items-eps-delta}, the lower bound on $k$ is always
finite even if $\tau=\infty$, which means an attacker might be able to reduce
the attack cost $J(\tilde D)$ to 0. This is in contrast to
$\epsilon$-differentially-private learner, where $J( \tilde D)=0$ can never
be achieved. The weaker $(\epsilon,\delta)$-differential privacy
guarantee gives weaker protection against attacks.



	\section{Data Poisoning Attacks on Private Learners}
\label{sec:attack}
The results in the previous section provide a theoretical bound on how effective
a data poisoning attack can be against a private learner. To evaluate how strong
the protection is in practice, we propose a range of attacks targeting general
differentially-private learners. Our adversary will \emph{modify} (not add or
delete) any continuous features (for classification or regression) and the
continuous labels (for regression) on at most $k$ items in the training set.
Restating the attack problem~\eqref{attack:opt}:
\begin{equation}\label{attack:opt2}
\begin{aligned}
\min_{ \tilde D\subset \B(D,k)} \quad & J(\tilde  D)=\E{b}{ C(\M(\tilde D, b))}
\end{aligned}
\end{equation}
This is a combinatorial problem---the attacker must pick $k$ items to modify. We
propose a two-step attack procedure.
\begin{enumerate}[step I)]
\item use heuristic methods to select $k$ items to poison.
\item apply stochastic gradient descent to reduce $J(\tilde D)$. 
\end{enumerate}
In step II), performing stochastic gradient descent could lead to poisoned data taking arbitrary features or labels.
However, differentially-private learners typically assume training examples are
bounded. To avoid trivially-detectable poisoned examples, we project poisoned
features and labels back to a bounded space after each iteration of SGD. We now
detail step II), assuming that the attacker has fixed $k$ items to poison. We'll
return to step I) later.

\subsection{SGD on Differentially-Private Victims (DPV)}
\label{DPV}
Our attacker uses stochastic gradient descent to minimize the attack cost $J(\tilde
D)$. We first show how to compute a stochastic gradient \textbf{with respect to
a training item}.
\begin{restatable}{proposition}{sgd}
\label{prop:sgd}
Assume $\M(\tilde D, b)$ is a differentiable function of $\tilde D$, then a
stochastic gradient of $J(\tilde D)$ w.r.t. a particular item $\tilde z_i$ is
$\frac{\partial C(\M(\tilde D, b))}{\partial \tilde z_i}$, where $b\sim\nu(b)$.
\end{restatable}
\begin{proof}
Note that
\begin{equation}
\E{b}{\frac{\partial C(\M(\tilde D, b))}{\partial \tilde z_i}}=\int\frac{\partial C(\M(\tilde D, b))}{\partial \tilde z_i}\nu(b)db.
\end{equation}
Under certain conditions (see appendix for the details), one can exchange the order of integration and differentiation, and we have
\begin{equation}
\E{b}{\frac{\partial C(\M(\tilde D, b))}{\partial \tilde z_i}}=\frac{d \int C(\M(\tilde D,b))\nu(b)db}{d \tilde z_i}=\frac{dJ(\tilde D)}{d \tilde z_i}.
\end{equation}
Thus $\frac{\partial C(\M(\tilde D, b))}{\partial \tilde z_i}$ is a stochastic gradient. 
\end{proof}


The concrete form of the stochastic gradient depends on the target private
learner $\M$. In the following, we derive the stochastic gradient for objective
perturbation and output perturbation in the context of parameter-targeting
attack with cost function $C(\tilde\theta)=\frac{1}{2}\|\tilde\theta-\theta^\prime\|^2$, where $\theta^\prime$ is the target model.

\subsubsection{Instantiating Attack on Objective Perturbation}
We first consider objective-perturbed private learners:
\begin{equation}\label{learner:obj}
\M(\tilde D,b)=\underset{\theta\in\Theta}\argmin\sum_{i=1}^n\ell(\theta, \tilde z_i)+\lambda \Omega(\theta)+b^\top \theta,
\end{equation}
where $\ell$ is some convex learning loss, $\Omega$ is a regularization, and
$\Theta$ is the model space. Note that the noise $b$ enters the objective via
linear product with the model $\theta$.  By Proposition~\ref{prop:sgd}, the
stochastic gradient is $\frac{\partial C(\M(\tilde D, b))}{\partial \tilde z_i
}=\frac{\partial C(\tilde\theta)}{\partial \tilde z_i }$, where
$\tilde\theta=\M(\tilde D, b)$ is the poisoned model. 
By the chain rule, we have 
\begin{equation}\label{sgd:obj}
\frac{\partial C(\tilde\theta)}{\partial \tilde z_i}=\left(\frac{\partial \tilde\theta}{\partial \tilde z_i}\right)^\top\frac{dC(\tilde\theta)}{d\tilde\theta}.
\end{equation}  
Note that $\frac{dC(\tilde\theta)}{d\tilde\theta}=\tilde\theta-\theta^\prime$. Next we focus on
deriving $\frac{\partial \tilde\theta}{\partial \tilde z_i}$.
Since~\eqref{learner:obj} is a convex optimization problem, the learned model $\tilde\theta$ must
satisfy the following Karush-Kuhn-Tucker (KKT) condition:
\begin{equation}
f(\tilde \theta,\tilde z_i) := \sum_{j=1}^n\frac{\partial\ell(\tilde\theta,\tilde z_j)}{\partial \tilde\theta}+\lambda \frac{d\Omega(\tilde\theta)}{d\tilde\theta}+b=0.
\end{equation}
By using the derivative of implicit function, we have $\frac{\partial \tilde\theta}{\partial \tilde z_i}=-(\frac{\partial f}{\partial \tilde\theta})^{-1}\frac{\partial  f}{\partial \tilde z_i}$.
We now give two examples where the base learner is logistic regression and ridge regression, respectively.

\paragraph{Attacking Objective-Perturbed Logistic Regression.}
In the case $\ell(\theta, \tilde z)=\log(1+\exp(-\tilde y\theta^\top \tilde x))$ and $\Omega(\theta)=\frac{1}{2}\|\theta\|^2$, learner~\eqref{learner:obj} is objective-perturbed logistic regression:
\begin{equation}\label{logit:obj-attacked}
\M(\tilde D, b)=\underset{\theta\in\Theta}\argmin\sum_{i=1}^n\log(1+\exp(-\tilde y_i\theta^\top \tilde x_i))+\frac{\lambda}{2}\|\theta\|^2+b^\top \theta,
\end{equation}
where $\Theta=\R^d$. Our attacker will only modify the features, thus $\tilde y_i=y_i$. 
We now derive stochastic gradient for $\tilde x_i$, which is
\begin{equation}
\frac{\partial C(\tilde \theta)}{\partial \tilde x_i}=(\frac{\partial \tilde\theta}{\partial \tilde x_i})^\top\frac{dC(\tilde\theta)}{d\tilde\theta}.
\end{equation}
To compute $\frac{\partial  \tilde\theta}{\partial  \tilde x_i}$, we use the KKT condition of~\eqref{logit:obj-attacked}:
\begin{equation}\label{obj-logit-KKT}
\lambda\tilde\theta-\sum_{j=1}^n\frac{y_j\tilde x_j}{1+\exp(y_j\tilde\theta^\top \tilde x_j)}+b=0.
\end{equation}
Let $f(\tilde \theta,\tilde x_i)=\lambda\tilde\theta-\sum_{j=1}^n\frac{y_j\tilde x_j}{1+\exp(y_j\tilde\theta^\top \tilde x_j)}+b$ and define $s_j=\exp(y_j\tilde\theta^\top\tilde x_j)$, then one can verify that
\begin{equation}
\frac{\partial f}{\partial\tilde\theta}=\lambda I+\sum_{j=1}^n\frac{s_j\tilde x_j\tilde x_j^\top}{(1+s_j)^2}.
\end{equation}
and
\begin{equation}
\frac{\partial f}{\partial \tilde x_i}=\frac{s_i\tilde x_i\tilde\theta^\top}{(1+s_i)^2}-\frac{y_i}{1+s_i}I.
\end{equation}
Note that $\frac{\partial \tilde\theta}{\partial \tilde x_i}=-(\frac{\partial f}{\partial \tilde\theta})^{-1}\frac{\partial  f}{\partial \tilde x_i}$ and $\frac{dC(\tilde\theta)}{d\tilde\theta}=\tilde\theta-\theta^\prime$, therefore the stochastic gradient of $\tilde x_i$ is
\begin{equation}\label{obj-logit-sgdx}
\frac{\partial C(\tilde \theta)}{\partial \tilde x_i}=(\frac{\partial \tilde\theta}{\partial \tilde x_i})^\top\frac{dC(\tilde\theta)}{d\tilde\theta}=\left(\frac{y_i}{1+s_i}I-\frac{s_i\tilde\theta\tilde x_i^\top}{(1+s_i)^2}\right)\left(\lambda I+\sum_{j=1}^n\frac{s_j\tilde x_j\tilde x_j^\top}{(1+s_j)^2}\right)^{-1}(\tilde\theta-\theta^\prime).
\end{equation}
Note that the noise $b$ enters into stochastic gradient through $\tilde\theta=\M(\tilde D, b)$.

\paragraph{Attacking Objective-Perturbed Ridge Regression.}
In the case $\ell(\theta, \tilde z)=\frac{1}{2}(\tilde y-\tilde x^\top\theta)^2$ and $\Omega(\theta)=\frac{1}{2}\|\theta\|^2$, learner~\eqref{learner:obj} is the objective-perturbed ridge regression:
\begin{equation}\label{ridge:obj-attacked}
\M(\tilde D, b)=\underset{\theta\in\Theta}\argmin \frac{1}{2}\|\tilde X\theta- \tilde y\|^2+\frac{\lambda}{2}\|\theta\|^2+b^\top \theta,\\
\end{equation}
where $\Theta=\{\theta \in \R^d: \|\theta\|_2\le \rho\}$, 
$\tilde X\in \R  ^{n\times d}$ is the feature matrix, and $\tilde y\in \R^n$ is the label
vector. Unlike logistic regression, the objective-perturbed ridge regression
requires the model space to be bounded (see e.g.~\cite{kifer2012private}). The
attacker can modify both the features and the labels.
We first compute stochastic gradient for $\tilde x_i$, which is
\begin{equation}
\frac{\partial C(\tilde \theta)}{\partial \tilde x_i}=(\frac{\partial \tilde\theta}{\partial \tilde x_i})^\top\frac{dC(\tilde\theta)}{d\tilde\theta}.
\end{equation}
Since $\frac{dC(\tilde\theta)}{d\tilde\theta}=\tilde\theta-\theta^\prime$, it suffices to compute $\frac{\partial \tilde\theta}{\partial \tilde x_i}$, for which we use the KKT condition of~\eqref{ridge:obj-attacked}:
\begin{equation}\label{KKT:ridge}
\left\{
\begin{aligned}
&(\tilde X^\top \tilde X+\lambda I)\tilde\theta -\tilde X^\top \tilde y+b+\mu\tilde\theta=0\\
& \frac{1}{2}\|\tilde\theta\|^2\le \frac{1}{2}\rho^2\\
& \mu\ge 0\\
& \mu(\|\tilde\theta\|^2-\rho^2)=0,
\end{aligned}
\right.
\end{equation}
where $\mu$ is the Lagrange dual variable for constraint $\frac{1}{2}\|\tilde\theta\|^2\le \frac{1}{2}\rho^2$. Let $f(\tilde\theta ,\tilde x_i,\mu)=(\tilde X^\top \tilde X+\lambda I)\tilde\theta -\tilde X^\top \tilde y+b+\mu\tilde\theta$. 
One can rewrite $f(\tilde\theta,\tilde x_i,\mu)$ as
\begin{equation}\label{f-ridge}
\begin{aligned}
f(\tilde\theta,\tilde x_i, \mu)&=(\tilde x_i\tilde x_i^\top+ \sum_{j\neq i} \tilde x_j\tilde x_j^\top +\lambda I)\tilde\theta-\tilde x_i \tilde y_i-\sum_{j\neq i} \tilde x_j \tilde y_j+b+\mu\tilde\theta.
\end{aligned}
\end{equation}
Then one can verify that
\begin{equation}
\frac{\partial f}{\partial \tilde x_i}=\tilde x_i\tilde\theta ^\top+(\tilde x_i^\top\tilde\theta-\tilde y_i) I.
\end{equation}
and
\begin{equation}
\frac{\partial f}{\partial \tilde \theta}=\tilde X^\top \tilde X+(\lambda+\mu) I .
\end{equation}
Thus the stochastic gradient of $\tilde x_i$ is
\begin{equation}\label{obj-ridge-sgdx}
\begin{aligned}
\frac{\partial C(\tilde\theta)}{\partial \tilde x_i}=\left(\tilde\theta \tilde x_i^\top+(\tilde x_i^\top \tilde\theta-\tilde y_i) I\right)\left(\tilde X^\top \tilde X+(\lambda+\mu) I\right)^{-1}(\theta^\prime-\tilde\theta).
\end{aligned}
\end{equation}
Next we derive the stochastic gradient of $\tilde y_i$. Now view $f$ in~\eqref{f-ridge} as a function of $\tilde\theta$, $\tilde y_i$ and $\mu$, then one can verify that
\begin{equation}
\frac{\partial f}{\partial \tilde y_i}=-\tilde x_i.
\end{equation}
and
\begin{equation}
\frac{\partial f}{\partial \tilde \theta}=\tilde X^\top \tilde X+(\lambda+\mu) I .
\end{equation}
Thus the stochastic gradient of $\tilde y_i$ is
\begin{equation}\label{obj-ridge-sgdy}
\frac{\partial C(\tilde\theta)}{\partial \tilde y_i}=\tilde x_i^\top\left(\tilde X^\top \tilde X+(\lambda+\mu) I\right)^{-1}(\tilde\theta -\theta^\prime).
\end{equation}
Note that the noise $b$ enters into the stochastic gradient through $\tilde\theta=\M(\tilde D, b)$.

\subsubsection{Instantiating Attack on Output Perturbation}
We now consider the output perturbation mechanism:
\begin{equation}\label{learner:out}
\M(\tilde D,b)=b+\underset{\theta\in\Theta}\argmin \left\{ \sum_{i=1}^n\ell(\theta,\tilde z_i)+\lambda \Omega(\theta) \right\}.
\end{equation}
The derivations of stochastic gradient are
similar to objective perturbation, thus we skip the details.  Again, we
instantiate on two examples where the base learner is logistic regression and
ridge regression, respectively.

\paragraph{Attacking Output-Perturbed Logistic Regression.}
The output-perturbed logistic regression takes the following form:
\begin{equation}\label{logistic:out}
\M(\tilde D,b)=b+\underset{\theta\in\R^d}\argmin\left\{\sum_{i=1}^n\log(1+\exp(-\tilde y_i\theta^\top \tilde x_i))+\frac{\lambda}{2}\|\theta\|^2 \right\},
\end{equation}
Let $s_j=\exp(y_j(\tilde\theta-b)^\top\tilde x_j)$, then the stochastic gradient of $\tilde x_i$ is
\begin{equation}
\frac{\partial C(\tilde\theta)}{\partial \tilde x_i}=\left(\frac{y_i}{1+s_i}I-\frac{s_i(\tilde\theta-b)\tilde x_i^\top}{(1+s_i)^2}\right)
\left(\lambda I+\sum_{j=1}^n\frac{s_j\tilde x_j\tilde x_j^\top}{(1+s_j)^2}\right)^{-1}(\tilde\theta-\theta^\prime).
\end{equation}

\paragraph{Attacking Output-Perturbed Ridge Regression.}
The output-perturbed ridge regression takes the following form
\begin{equation}\label{ridge:out}
\M(\tilde D,b)=b+\underset{\theta\in\Theta}\argmin \left\{ \frac{1}{2}\|\tilde X\theta-\tilde y\|^2+\frac{\lambda}{2}\|\theta\|^2\right\},
\end{equation}
where $\Theta=\{\theta:\|\theta\|_2\le \rho\}$. The KKT condition of~\eqref{ridge:out} is
\begin{equation}
\left\{
\begin{aligned}
&(\tilde X^\top \tilde X+\lambda I)(\tilde\theta-b) -\tilde X^\top \tilde y+\mu(\tilde\theta-b)=0\\
& \frac{1}{2}\|\tilde\theta-b\|^2\le \frac{1}{2}\rho^2\\
& \mu\ge 0\\
& \mu(\|\tilde\theta-b\|^2-\rho^2)=0.
\end{aligned}
\right.
\end{equation}
Then one can verify that the stochastic gradients of $\tilde x_i$ and $\tilde y_i$ are
\begin{equation}
\begin{aligned}
\frac{\partial C(\tilde\theta)}{\partial \tilde x_i}=&\left( (\tilde\theta-b) \tilde x_i^\top+\tilde x_i^\top (\tilde\theta-b) I-\tilde y_iI \right) \left( \tilde X^\top \tilde X+(\lambda +\mu) I)\right)^{-1}(\theta^\prime-\tilde\theta).\end{aligned}
\end{equation}
\begin{equation}
\frac{\partial C(\tilde\theta)}{\partial \tilde y_i}=\tilde x_i^\top \left( \tilde X^\top \tilde X+(\lambda+\mu) I\right)^{-1}(\tilde\theta -\theta^\prime).
\end{equation}

\subsection{SGD on Surrogate Victims (SV)}
\label{SV}
We also consider an alternative, simpler way to perform step II). 
When $b=\mathbf{0}$, the differentially-private learning algorithms revert to the base learner which returns a deterministic model $\M(\tilde D, \mathbf{0})$.
The adversary can simply attack the base learner (without $b$) as a surrogate for
the differentially-private learner $\M$ (with $b$)
by solving the following problem:
\begin{equation}\label{obj:attack-vanilla}
\min_{\tilde D}C(\M(\tilde D, \mathbf{0})).
\end{equation}
Since the objective is deterministic, we can work with its gradient
rather than its stochastic gradient, plugging in $b=\mathbf{0}$ into all
derivations in section~\ref{DPV}. Note that the attack found
by~\eqref{obj:attack-vanilla} will still be evaluated w.r.t.
differentially-private learners in our experiments.

\subsection{Selecting Items to Poison}
Now, we return to step I) of our attack: how can we select the $k$ items for
poisoning? We give two heuristic methods: shallow and deep selection.

\textbf{Shallow selection} selects top-$k$ items with the largest initial gradient norm $\|\frac{\partial J(\tilde D)}{ \partial \tilde z_i}\mid_{\tilde D=D}\|$.
Intuitively, modifying these items will reduce the attack cost the most, at
least initially.  The precise gradients to use during selection depend on the
attack in step II). When targeting differentially-private learners directly,
computing the gradient of $J(\tilde D)$ is difficult. We use Monte-Carlo sampling to approximate the
gradient $g_i$ of item $z_i$: $g_i
\approx\frac{1}{m}\sum_{s=1}^m\frac{\partial}{\partial \tilde z_i}C(\M(\tilde
D,b_s))\mid_{\tilde D=D}$, where $b_s$ are random samples of privacy parameters.
Then, the attacker picks the $k$ items with the largest $\|g_i\|$ to poison. When
targeting surrogate victim, the objective is $C(\M(\tilde D, \mathbf{0}))$, thus
the attacker computes the gradient of $C(\M(\tilde D, \mathbf{0}))$ of each
clean item $z_i$: $g_i = \frac{d}{d\tilde z_i}C(\M(\tilde
D,\mathbf{0}))\mid_{\tilde D=D}$, and then picks those $k$ items with the
largest $\|g_i\|$. 

\textbf{Deep selection} selects items by estimating the influence of an item on
the final attack cost. When targeting a differentially-private victim directly,
the attacker first solves the following \textbf{relaxed attack} problem:
$
\min_{ \tilde D} J( \tilde D)+\alpha R(\tilde D).
$
Importantly, here $\tilde D$ allows changes to all training items, not just $k$ of them.
$R(\tilde D)$ is a regularizor penalizing data modification with
weight $\alpha>0$. We take $R(\tilde D)=\sum_{i=1}^n
r(\tilde z_i)$, where $r(\tilde z_i)$ is the distance between the poisoned item
$\tilde z_i$ and the clean item $z_i$. We define $r(\tilde
z_i)=\frac{1}{2}\|\tilde x_i-x_i\|^2$ for logistic regression and $r(\tilde
z_i)=\frac{1}{2}\|\tilde x_i-x_i\|^2+\frac{1}{2}(\tilde y_i-y_i)^2$ for ridge
regression. After solving the relaxed attack problem with SGD, the attacker evaluates the
amount of change $r(\tilde z_i)$ for all training items and pick the top $k$ to poison.  When
targeting a surrogate victim, the attacker can instead solve the following
relaxed attack first:
$
\min_{ \tilde D} C(\M(\tilde D,\mathbf{0}))+\alpha R(\tilde D),
$
and then select $k$ top items. 

In summary, we have four attack algorithms based on combinations of
methods used in step I) and step II): shallow-SV, shallow-DPV, deep-SV and
deep-DPV. In the following, we evaluate the performance of these four
algorithms.

	\section{Experiments}
We now evaluate our attack with experiments, taking objective-perturbed
learners~\cite{kifer2012private} and output-perturbed
learners~\cite{chaudhuri2011differentially} as our victims. Through the
experiment, we fix a constant step size $\eta=1$ for (stochastic) gradient descent. After each iteration of SGD, we project poisoned items to ensure feature norm is at most 1, and label stays in $[-1,1]$.
To perform shallow selection when targeting DPV, we draw $10^3$ samples of parameter $b$ to evaluate the gradient of each clean item. 
To perform deep selection, we use $\alpha=10^{-4}$ to solve the corresponding relaxed attack problem.

\subsection{Attacks Intelligently Modify Data}
\label{sec:1D}
As a first example, we use label-aversion attack to illustrate that our attack modifies data intelligently.
The victim is an objective-perturbed learner for $\epsilon$-differentially private logistic regression, with $\epsilon=0.1$ and regularizer $\lambda=10$. 
The training set contains $n=21$ items uniformly sampled in the interval $[-1,1]$ with labels $y_i=\ind{x_i\ge 0}$, see Figure~\ref{fig:combo}(\subref{fig:traj1D}).
To generate the evaluation set,  we construct $m=21$ evenly-spaced items in the interval $[-1,1]$, i.e., $x_i^*=0.1i, -10\le i\le10$ and labeled as $y_i^*=\ind{x_i^*\ge 0}$. 
The cost function is defined as $C(\tilde \theta)=-\frac{1}{m}\sum_{i=1}^m\log(1+\exp(-y_i^* {x_i^*}^\top \tilde\theta))$.
To achieve small attack cost $J$, a negative model $\tilde \theta<0$ is
desired.\footnote{%
Differentially-private machine learning typically considers homogeneous learners, thus the 1D model is just a scalar.}
We run deep-DPV with $k=n$ (attacker can modify all items) 
\footnote{Since $k=n$, the attacker does not need to select items in step I).} for $T=300$
iterations.  In Figure~\ref{fig:combo}(\subref{fig:traj1D}), we show the
position of the poisoned items on the vertical axis as the SGD iteration, $t$,
grows. As expected, the attacker flips the positions of positive (blue) items
and negative (red) items. As a result, our attack causes the learner to find a
model with $\tilde\theta<0$.

\subsection{Attacks Can Achieve Different Goals}
\label{exp:ADG}
We now study a 2D example to show that our attack is effective for different attack goals. 
The victim is the same as in section~\ref{sec:1D}.
The clean training set in Figure~\ref{fig:combo}(\subref{fig:2D-training})  contains $n=317$ items uniformly sampled in a unit sphere with labels $y_i=\ind{x_i^\top\theta^*\ge 0}$ where $\theta^*=(1,1)$.
We now describe the attack settings for different attack goals.

\paragraph{Label-aversion attack.} 
We generate an evaluation set containing $m=317$ grid points $(x_i^*,y_i^*)$ lying in a unit sphere and labeled by a vertical decision boundary, see Figure~\ref{fig:combo}(\subref{fig:2D-evaluation}). 
The cost function is defined as $C(\tilde\theta)=-\frac{1}{m}\sum_{i=1}^m\log(1+\exp(-y_i^* {x_i^*}^\top \tilde\theta))$. 

\paragraph{Label-targeting attack.} 
The evaluation set remains the same as in the label-aversion attack. 
The cost function is defined as $C(\tilde\theta)=\frac{1}{m}\sum_{i=1}^m\log(1+\exp(-y_i^* {x_i^*}^\top \tilde\theta))$.

\paragraph{Parameter-targeting attack.} 
We first train vanilla logistic regression on the evaluation set to obtain the target model $\theta^\prime=(2.6,0)$, then define cost function as $C(\tilde\theta)=\frac{1}{2}\|\tilde\theta-\theta^\prime\|^2$.
\begin{figure}[h]
	\centering
	\begin{subfigure}[t]{.3\textwidth}
		\centering
		\includegraphics[width=0.9\textwidth, height=0.72\textwidth]{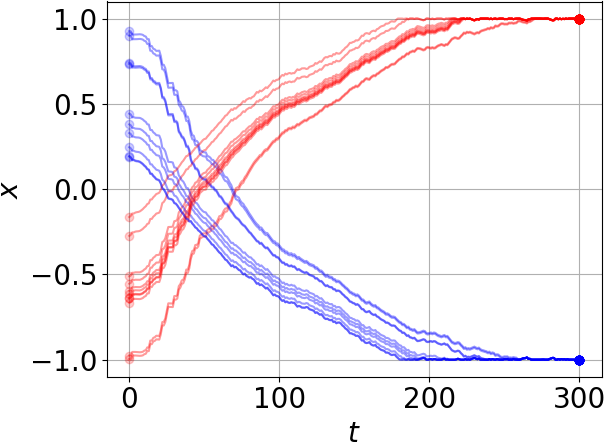}
                \caption{poisoning trajectory}
                \label{fig:traj1D}
	\end{subfigure}
	\begin{subfigure}[t]{.3\textwidth}
		\centering
		\includegraphics[width=0.8\textwidth]{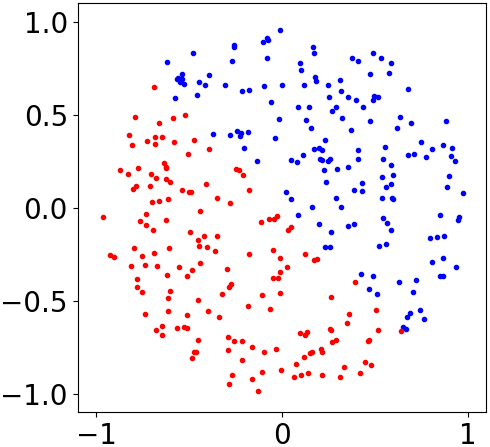}
		\caption{training set}
		\label{fig:2D-training}
	\end{subfigure}%
	\begin{subfigure}[t]{.3\textwidth}
		\centering
		\includegraphics[width=0.8\textwidth]{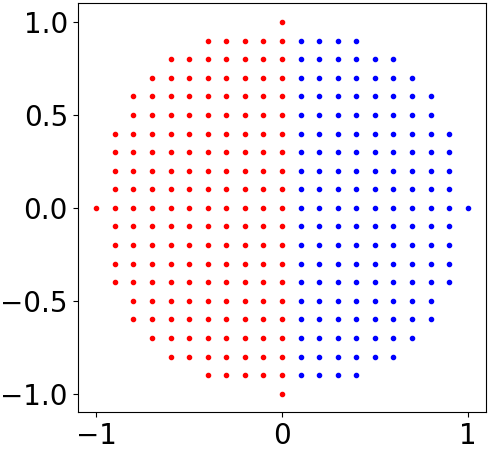}
	 	\caption{evaluation set}
		\label{fig:2D-evaluation}
	\end{subfigure}%
	\caption{(a) 1D example (b, c) 2D example}
	\label{fig:combo}
\end{figure}

We run deep-DPV with $k=n$ for $T=5\times 10^3$ iterations. 
In Figure~\ref{fig:traj2D}(\subref{fig:labelaversion_2Dtraj})-(\subref{fig:parametertargeting_2Dtraj}), we show the poisoning trajectories for the three attack goals respectively.  Each translucent point is an original training point, and the connected solid point is its final position after attack.  The curve connecting them shows the trajectory as the attack algorithm gradually poisons the data.

In label-aversion attack the attacker aims to maximize the logistic loss on the evaluation set. It ends up moving positive (negative) items to the left (right), so that the poisoned data deviates from the evaluation set. 

In contrast, the label-targeting attack tries to minimize the logistic loss, thus items are moved to produce a poisoned data aligned with the evaluation set. However, the attack does not reproduce exactly the evaluation set. To understand it, we compute the model learnt by vanilla logistic regression on the evaluation set and the poisoned data, which are (2.60,0) and (2.94,0.01). Note that both models can predict the right label on the evaluation set, but the latter has larger norm, which leads to smaller attack cost, thus our result is a better attack.

In parameter-targeting attack, we compute the model learnt by vanilla logistic regression on the clean data and the poisoned data, which are (1.86,1.85) and $(2.21, 0.04)$. Note that the attack pushes the model closer to the target model (2.6,0). Again, the attack does not reproduce exactly the evaluation set, which is because the goal is to minimize the attack cost $J=\E{b}{C(\tilde\theta)}$ rather than $C(\tilde\theta)$. To see that, we evaluate the attack cost $J$ (as explained below) on the evaluation set and the poisoned data, and the values are 2.12 and 1.08 respectively, thus our result is a better attack. 

To quantify the attack performance, we evaluate the attack cost $J(\tilde D)$ via Monte-Carlo sampling. 
We run private learners $T_e=10^3$ times, each time with a random $b_s$ to produce a cost $C_s=C(\M(\tilde D,b_s))$. 
Then we average to obtain an estimate of the cost $J(\tilde D)\approx\frac{1}{T_e}\sum_{s=1}^{T_e}C_s$. 
In Figure~\ref{fig:traj2D}(\subref{fig:labelaversion_2DJt})-(\subref{fig:parametertargeting_2DJt}),
we show the attack cost $J(\tilde D)$ and its 95\% confidence interval (green, $\pm 2*$stderr) up to $t=10^3$ iterations. $J(\tilde D)$ decreases in all plots, showing that our attack is effective for all attack goals.

 \begin{figure}
 \centering
	 \begin{subfigure}[t]{.3\textwidth}
		\centering
		\includegraphics[width=0.8\textwidth]{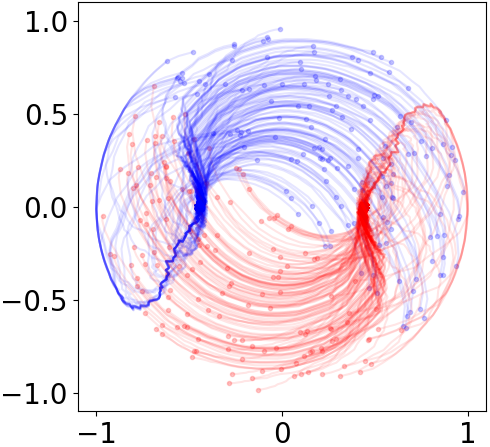} 
		\caption{label-aversion}
		\label{fig:labelaversion_2Dtraj}
	\end{subfigure}%
	\begin{subfigure}[t]{.3\textwidth}
		\centering
		\includegraphics[width=0.8\textwidth]{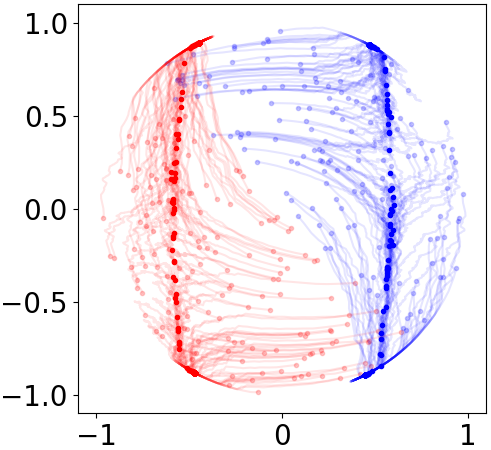} 
		\caption{label-targeting}
		\label{fig:labeltargeting_2Dtraj}
	\end{subfigure}%
	 \begin{subfigure}[t]{.3\textwidth}
		\centering
		\includegraphics[width=0.8\textwidth]{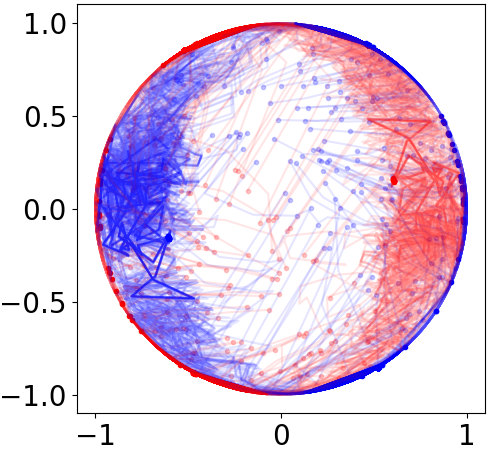} 
		\caption{parameter-targeting}
		\label{fig:parametertargeting_2Dtraj}
	\end{subfigure}%
	
	 \begin{subfigure}[t]{.3\textwidth}
		\centering
		\includegraphics[width=0.8\textwidth]{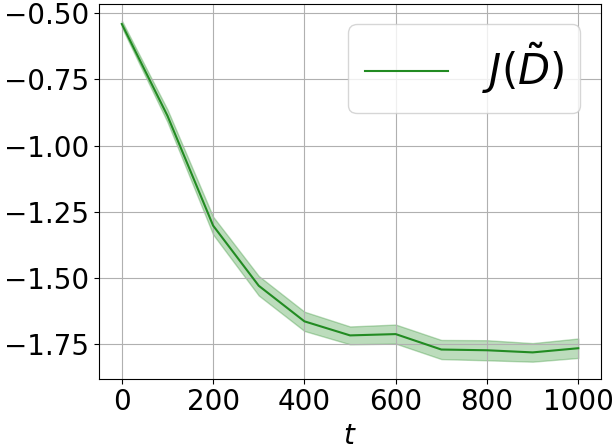} 
		\caption{label-aversion}
		\label{fig:labelaversion_2DJt}
	\end{subfigure}%
	\begin{subfigure}[t]{.3\textwidth}
		\centering
		\includegraphics[width=0.8\textwidth]{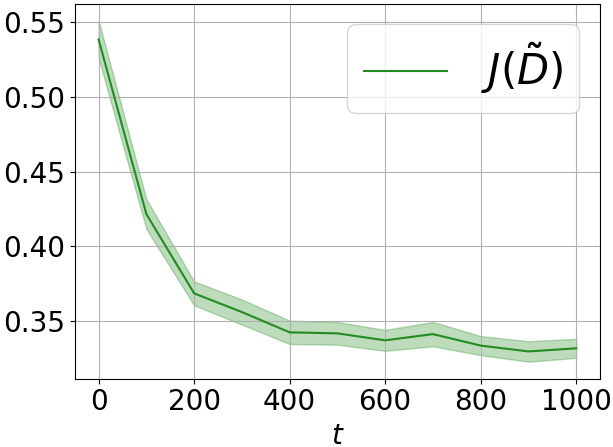} 
		\caption{label-targeting}
		\label{fig:labeltargeting_2DJt}
	\end{subfigure}%
	 \begin{subfigure}[t]{.3\textwidth}
		\centering
		\includegraphics[width=0.8\textwidth]{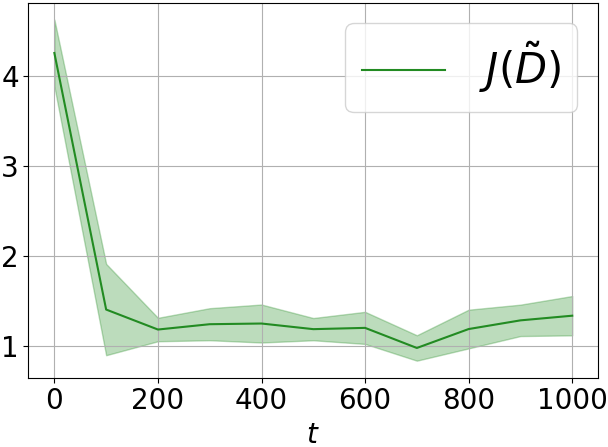} 
		\caption{parameter-targeting}
		\label{fig:parametertargeting_2DJt}
	\end{subfigure}%
	\caption{poisoning trajectories in 2D.}
	\label{fig:traj2D}
\end{figure}

\subsection{Attacks Become More Effective as $k$ Increases}
\label{sec:k}
The theoretical protection provided by differential privacy weakens as the adversary is allowed to poison more training items.
In this section, we show that our attacks indeed become more effective as the number of poisoned items $k$ increases.
The data set is the same as in section~\ref{exp:ADG}.
We run our four attack algorithms with $k$ from 20 to 100 in steps of 20. For each $k$, we first select $k$ items by using shallow or deep selection, and then poison the selected items by running
(stochastic) gradient descent for $T=5\times 10^3$ iterations. After attack, we estimate the attack cost $J(\tilde D)$ on $T_e=2\times 10^3$ samples.
Figure~\ref{fig:Jk2D}(\subref{fig:labelaversion_2DJk})-(\subref{fig:parametertargeting_2DJk}) show that the attack cost $J(\tilde D)$ decreases as $k$ grows, indicating better attack performance.
We also show the poisoning trajectory produced by deep-DPV with $k=10$ in Figure~\ref{fig:Jk2D}(\subref{fig:labelaversion_traj10})-(\subref{fig:parametertargeting_traj10}). The corresponding attack costs $J(\tilde D)$ are $-0.60,0.49$ and $3.43$ respectively. Compared to that of $k=n$, $-1.79$, $0.32$ and $1.20$, we see that poisoning only 10 items is not enough to reduce the attack cost $J(\tilde D)$ significantly, although the attacker is making effective modification.
 \begin{figure}
	\centering
	\begin{subfigure}[t]{.3\textwidth}
		\centering
		\includegraphics[width=0.8\textwidth]{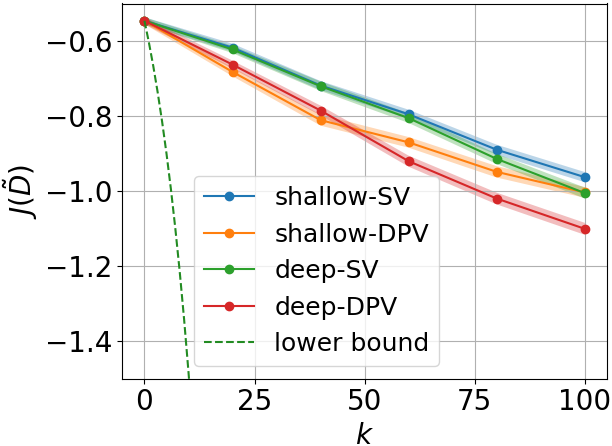}
		\caption{label-aversion}
		\label{fig:labelaversion_2DJk}
	\end{subfigure}
	\begin{subfigure}[t]{.3\textwidth}
		\centering
		\includegraphics[width=0.8\textwidth]{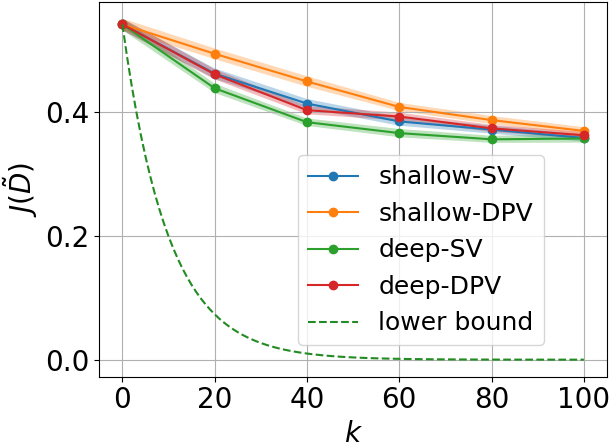} 
		\caption{label-targeting}
		\label{fig:labeltargeting_2DJk}
	\end{subfigure}%
	\begin{subfigure}[t]{.3\textwidth}
		\centering
		\includegraphics[width=0.8\textwidth]{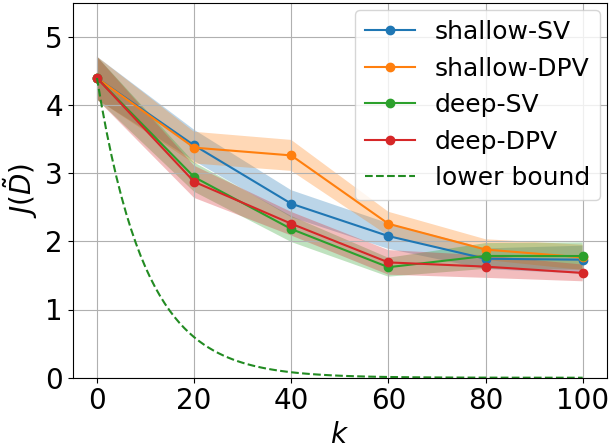} 
		\caption{parameter-targeting}
		\label{fig:parametertargeting_2DJk}
	\end{subfigure}%
	
	\centering
	\begin{subfigure}[t]{.3\textwidth}
		\centering
		\includegraphics[width=0.8\textwidth]{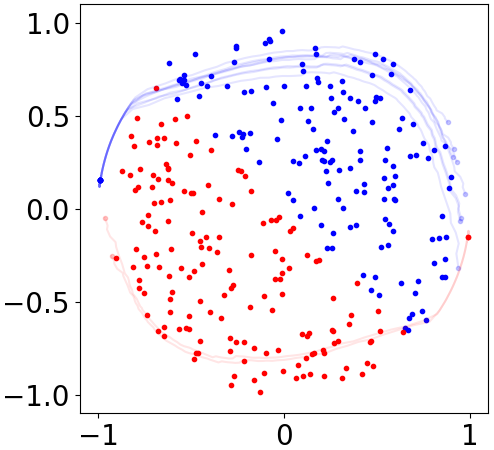}
		\caption{label-aversion}
		\label{fig:labelaversion_traj10}
	\end{subfigure}
	\begin{subfigure}[t]{.3\textwidth}
		\centering
		\includegraphics[width=0.8\textwidth]{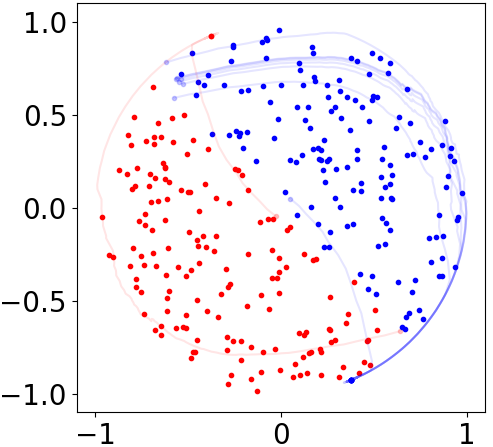} 
		\caption{label-targeting}
		\label{fig:labeltargeting_traj10}
	\end{subfigure}%
	\begin{subfigure}[t]{.3\textwidth}
		\centering
		\includegraphics[width=0.8\textwidth]{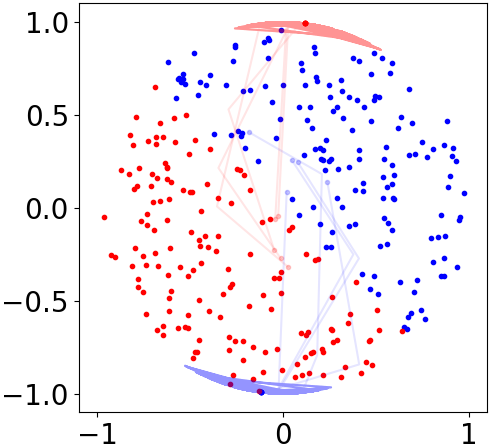} 
		\caption{parameter-targeting}
		\label{fig:parametertargeting_traj10}
	\end{subfigure}%
	\caption{(a-c) the attack cost $J(\tilde D)$ decreases as $k$ grows. (d-f) the sparse attack trajectories for $k=10$}
	\label{fig:Jk2D}
\end{figure}

\subsection{Attacks Effective on Both Privacy Mechanisms}
We now show experiments on real data to illustrate that our attacks are effective on both objective and output perturbation mechanisms. 
We focus on label-targeting attack in this section.

The first real data set is vertebral column from the UCI Machine
Learning Repository~\cite{Dua:2017}.  This data set contains 6 features of 310
orthopaedic patients, and the task is to predict if the vertebra of any patient
is abnormal (positive class) or not (negative class).  We normalize the features
so that the norm of any item is at most 1.
Since this is a classification task, we use private logistic regression with $\epsilon=0.1$ and $\lambda=10$ as the victim.
To generate the evaluation set, we randomly pick one positive item and find its
$10$ nearest neighbours within the positive class, and set  
$y_i^*=-1$ for all 10.  Intuitively, this attack targets a small cluster of abnormal patients and
the goal is to mislead the learner to classify these patients as normal. The
other experimental parameters remain the same as in section~\ref{sec:k}.
Note that in order to push the prediction on $x_i^*$ toward $y_i^*$, the attacker requires $y_i^*{x_i^*}^\top\tilde\theta>0$, which indicates $C(\tilde\theta)=\log(1+\exp(-y_i^*{x_i^*}^\top\tilde\theta))<\log 2\approx 0.69$, so should $J(\tilde D)$. As shown in Figure~\ref{fig:column}, the attacker indeed reduces the attack cost $J(\tilde D)$ below 0.69 for both privacy mechanisms, thus our attack is successful.
\begin{figure}[h]
	\centering
	\begin{subfigure}[t]{.35\textwidth}
		\centering
		\includegraphics[width=0.9\textwidth]{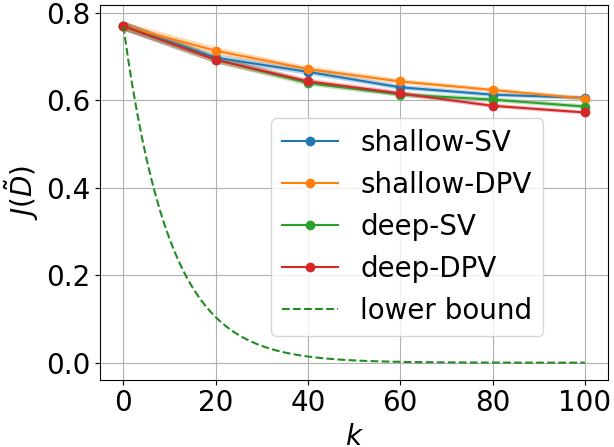}
		\caption{objective perturbation}
		\label{fig:column-obj}
	\end{subfigure}
	\begin{subfigure}[t]{.35\textwidth}
		\centering
		\includegraphics[width=0.9\textwidth]{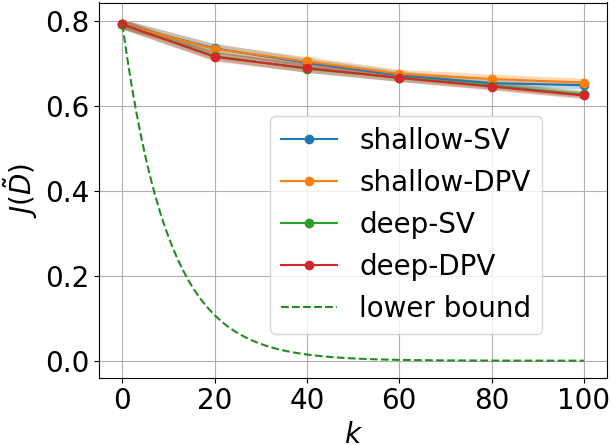} 
		\caption{output perturbation}
		\label{fig:column-out}
	\end{subfigure}%
	\caption{attack on objective and output-perturbed logistic regression.}
	\label{fig:column}
\end{figure}

The second data set is red wine quality from UCI. 
The data set contains 11 features of 1598 wine samples, and the task is predict the wine quality, a number between 0 and 10. 
We normalize the features so that the norm of any item is at most 1. 
Labels are also normalized to ensure the value is in $[-1,1]$.
The victim is private ridge regression with $\epsilon=1$, $\lambda=10$, and bound on model space $\rho=2$. 
To generate the evaluation set, we pick one item $x_i^*$ with the smallest quality value, and then set the target label to be $y_i^*=1$.
The cost function is defined as $C(\theta)=\frac{1}{2}({x_i^*}^\top \theta-y_i^*)^2$.
The other parameters of attack remain the same as in section~\ref{sec:k}.
This attack aims to force the learner to predict a low-quality wine as having a high quality. 
Figure~\ref{fig:wine} shows the results on objective perturbation and output perturbation respectively. 
Note that the attack cost can be effectively reduced even if the attacker only poisons $100/1598\approx6.3\%$ of the whole data set.

\begin{figure}[h]
	\centering
	\begin{subfigure}[t]{.35\textwidth}
		\centering
		\includegraphics[width=0.9\textwidth]{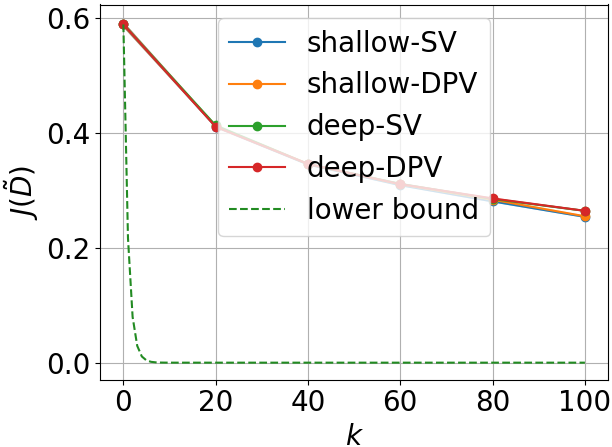}
		\caption{objective perturbation}
		\label{fig:wine-obj}
	\end{subfigure}
	\begin{subfigure}[t]{.35\textwidth}
		\centering
		\includegraphics[width=0.9\textwidth]{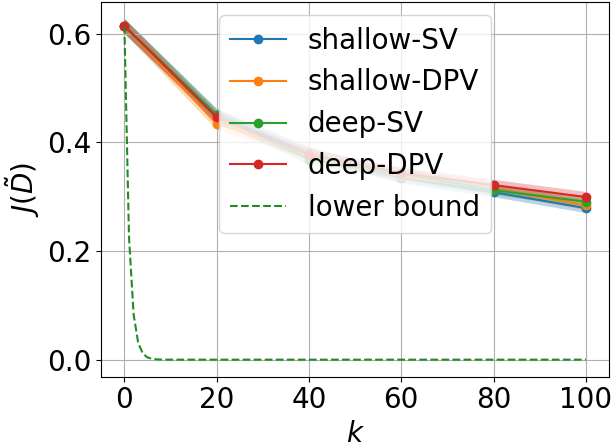} 
		\caption{output perturbation}
		\label{fig:wine-out}
	\end{subfigure}%
	\caption{attack on objective and output-perturbed ridge regression.}
	\label{fig:wine}
\end{figure}

Given experimental results above, we make another two observations here. First, deep-DPV is in  general the most effective method among four attack algorithms, see e.g. Figure~\ref{fig:Jk2D}(\subref{fig:labelaversion_2DJk}), (\subref{fig:parametertargeting_2DJk}). Second, although deep-DPV is effective, there remains a gap between the attack performance and the theoretical lower bound, as is shown in all previous plots. One potential reason is that the lower bound is derived purely based on the differential privacy property, thus does not take into account the learning procedure of the victim. Analysis oriented toward specific victim learners may result in a tighter lower bound. Another potential reason is that our attack might not be effective enough. How to close the gap is left as future work.

\subsection{Attack Is Easier with Weaker Privacy}
Finally, we show that the attack is more effective as the privacy guarantee
becomes weaker, i.e., as we increase $\epsilon$. 
To illustrate, we run experiments on the data set in section~\ref{exp:ADG}, where we fix $k=100$ and vary $\epsilon$ from $0.1$ to $0.5$.
In Figure~\ref{fig:eps}, we show the attack cost $J(\tilde D)$ against $\epsilon$. 
Note that the attack cost $J(\tilde D)$ approaches the lower bound as $\epsilon$ grows. 
This means weaker privacy guarantee (smaller $\epsilon$) leads to easier attacks. 

\begin{figure}[H]
        \centering
                \includegraphics[width=0.35\textwidth]{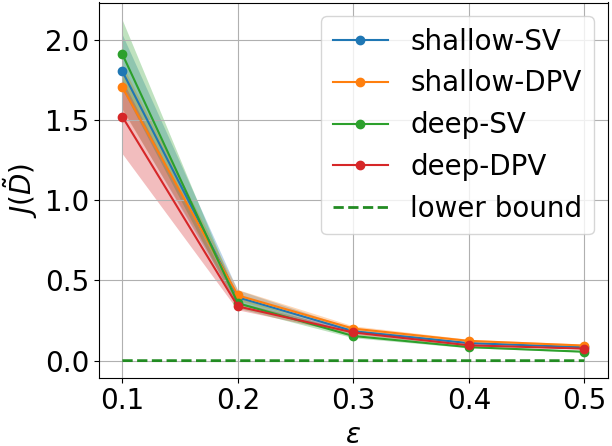}
                \caption{Attacker can reduce cost $J(\tilde D)$ more as $\epsilon$ grows.}
                \label{fig:eps}
\end{figure}

	\section{Conclusion and Future Work}
We showed that differentially private learners are provably resistant to data
poisoning attacks, with the protection degrading exponentially as the attacker
poisons more items. Then, we proposed attacks that can effectively poison
differentially-private learners, and demonstrated the attack performance on a variety of
privacy mechanisms and learners with both synthetic and real data.  While the
attacks are effective, there remains a gap between the theoretical lower bound
and the empirical performance of our attacks. This could be because the lower
bound is loose, or our attack is not effective enough; closing the gap remains
future work. Taken together, our study is a first step towards understanding the
strengths and weaknesses of differential privacy as a defensive measure against
data poisoning.

	\subsubsection*{\large Acknowledgements}
We thank Steve Wright for helpful discussions.
This work is supported in part by NSF
1545481, 1561512, 1623605, 1704117, 1836978, the MADLab AF Center of Excellence
FA9550-18-1-0166, a Facebook TAV grant, and the University of Wisconsin.
	\bibliographystyle{plainnat}
	\bibliography{ref}
	\section*{Appendix: Regularity Condition for Differentiation-Integration Exchange}
The following theorem characterizes a sufficient condition when the order of differentiation and integration can be exchanged (restated from Theorem 2.4.3 in~\cite{casella2002statistical}), which is called the regularity condition.
\begin{restatable}{theorem}{regular}
\label{thm:regular}
Suppose $h(z,b)$ is a differentiable function of $z$. If for each $z$, and there exist a function $g(z,b)$ and a constant $d>0$ that satisfy
\begin{enumerate}[i).]
\item $\|\frac{\partial}{\partial z}h(z, b)\mid_{z=z^\prime}\|\le g(z,b)$, for all $b$ and all $z^\prime$ such that $\|z^\prime-z\|\le d$.
\item $\int g(z,b)db<\infty$.
\end{enumerate}
Then for any $z$, we have
\begin{equation}
\frac{d}{dz}\int h(z,b)db=\int \frac{\partial}{\partial z}h(z,b)db.
\end{equation}
\end{restatable}

Note that in condition $i)$, the constant $d$ could depend on $z$ implicitly. Proposition~\ref{prop:sgd} is then a straight-forward corollary of the above theorem by taking function $h$ as $h(\tilde z_i,b)=C(\M(\tilde D, b))\nu(b)$. The $\nu(b)$ is usually chosen as the Laplace distribution $\nu(b)=\psi e^{-\|b\|}$ for some constant $\psi$, thus from now on we consider  $h(\tilde z_i,b)=C(\M(\tilde D, b))\psi e^{-\|b\|}$. We also require that the poisoned features and labels satisfy $\|\tilde x_i\|\le 1$ and $|\tilde y_i|\le 1$, to conform with the standard assumption that data must be bounded in differentially private machine learning. Next we use parameter-targeting attack as an example to illustrate that the regularity condition always holds in our attack problem.
\begin{enumerate}
\item Attacking objective-perturbed logistic regression.
\label{regularity:exp1}

In our case, the attacker only modifies the feature, thus we denote $h(\tilde x_i,b)=C(\M(\tilde D, b))\psi e^{-\|b\|}$. By~\eqref{obj-logit-sgdx},
\begin{equation}\label{app:derivative}
\frac{\partial C(\tilde \theta)}{\partial \tilde x_i}=\left(\frac{y_i}{1+s_i}I-\frac{s_i\tilde\theta\tilde x_i^\top}{(1+s_i)^2}\right)\left(\lambda I+\sum_{j=1}^n\frac{s_j\tilde x_j\tilde x_j^\top}{(1+s_j)^2}\right)^{-1}(\tilde\theta-\theta^\prime),
\end{equation}
where $s_j=\exp(y_j\tilde\theta^\top\tilde x_j)$ and  $\tilde\theta=\M(\tilde D, b)$. We let $d=1$ in condition $i)$. Note that $\|\tilde x_i\|\le 1$, thus for any $\tilde x_i^\prime$ such that $\|\tilde x_i^\prime-\tilde x_i\|\le d$, by triangle inequality we have $\|\tilde x_i^\prime\|\le \|\tilde x_i\|+d=2$. 
Let $s_i^\prime=\exp(y_i\tilde\theta^\top\tilde x_i^\prime)$. 
By~\eqref{app:derivative}, we have
\begin{equation}\label{app:intermediate}
\begin{aligned}
\|\frac{\partial C(\tilde \theta)}{\partial \tilde x_i}\mid_{\tilde x_i=\tilde x_i^\prime}\|&\le \|\frac{y_i}{1+s_i^\prime}I-\frac{s_i^\prime\tilde\theta \tilde x_i^{\prime^\top}}{(1+s_i^\prime)^2}\|\cdot \|\left(\lambda I+\sum_{\substack{1\le j\le n\\j\neq i}}\frac{s_j\tilde x_j\tilde x_j^\top}{(1+s_j)^2}+\frac{s_i^\prime\tilde x_i^\prime\tilde x_i^{\prime^\top}}{(1+s_i^\prime)^2}\right)^{-1}\|\cdot\|\tilde\theta-\theta^\prime\|\\
&\le \left(\|\frac{y_i}{1+s_i^\prime}I\|+\|\frac{s_i^\prime\tilde\theta\tilde x_i^{\prime^\top}}{(1+s_i^\prime)^2}\|\right)\frac{1}{\lambda}(\|\tilde\theta\|+\|\theta^\prime\|)\\
&= \frac{1}{\lambda}\left(\frac{1}{1+s_i^\prime}+\frac{s_i^\prime}{(1+s_i^\prime)^2}\|\tilde\theta\|\|\tilde x_i^\prime\|\right)(\|\tilde\theta\|+\|\theta^\prime\|)\\
&\le \frac{1}{\lambda}(1+\frac{1}{2}\|\tilde\theta\|)(\|\tilde\theta\|+\|\theta^\prime\|)\quad (\text{since } 1+s_i^\prime\ge 2\sqrt{s_i^\prime} \text{ and } \|\tilde x_i^\prime\|\le 2).
\end{aligned}
\end{equation}
Now we upper bound $\|\tilde \theta\|$. By~\eqref{obj-logit-KKT},
\begin{equation}\label{obj-logit-modelbound}
\begin{aligned}
\|\tilde\theta\|&=\frac{1}{\lambda}\|\sum_{\substack{1\le j\le n\\j\neq i}}\frac{y_j\tilde x_j}{1+s_j}+\frac{y_i\tilde x_i^\prime}{1+s_i^\prime}-b\|
\le\frac{1}{\lambda}(\sum_{\substack{1\le j\le n\\j\neq i}}\|\frac{y_j\tilde x_j}{1+s_j}\|+\|\frac{y_i\tilde x_i^\prime}{1+s_i^\prime}\|+\|b\|)\\
&\le\frac{1}{\lambda}(n-1+2+\|b\|)=\frac{1}{\lambda}(n+1+\|b\|).
\end{aligned}
\end{equation}
Plugging~\eqref{obj-logit-modelbound} back to~\eqref{app:intermediate}, for all $\tilde x_i^\prime$ such that $\|\tilde x_i^\prime-\tilde x_i\|\le 1$, we have
\begin{equation}
\|\frac{\partial C(\tilde \theta)}{\partial \tilde x_i}\mid_{\tilde x_i=\tilde x_i^\prime}\|\le  \frac{1}{\lambda}(1+\frac{n+1+\|b\|}{2\lambda})(\frac{n+1+\|b\|}{\lambda}+\|\theta^\prime\|)\le C_1\|b\|^2,
\end{equation}
for some constant $C_1>0$. Therefore
\begin{equation}
\|\frac{\partial h(\tilde x_i,b)}{\partial \tilde x_i}\mid_{\tilde x_i=\tilde x_i^\prime}\|=\|\frac{\partial C(\tilde \theta)}{\partial \tilde x_i}\mid_{\tilde x_i=\tilde x_i^\prime}\|\psi e^{-\|b\|}\le C_1\psi \|b\|^2e^{-\|b\|}.
\end{equation}
One can define $g(\tilde x_i,b)=C_1\psi \|b\|^2e^{-\|b\|}$, which is in fact a function of only $b$, to satisfy condition $i)$ in~\thmref{thm:regular}. Moreover, one can show that $\int g(\tilde x_i,b)db<\infty$, thus condition $ii)$ is also satisfied. Therefore we can exchange the order of differentiation and integration.

\item Attacking objective-perturbed ridge regression.

We separately show that $h(\tilde x_i,b)$ and $h(\tilde y_i,b)$ satisfy the conditions in~\thmref{thm:regular}.

First we consider $h(\tilde x_i, b)$. By~\eqref{obj-ridge-sgdx},
\begin{equation}
\begin{aligned}
\frac{\partial C(\tilde\theta)}{\partial \tilde x_i}=\left(\tilde\theta \tilde x_i^\top+(\tilde x_i^\top \tilde\theta-\tilde y_i) I\right)\left(\tilde X^\top \tilde X+(\lambda+\mu) I\right)^{-1}(\theta^\prime-\tilde\theta).
\end{aligned}
\end{equation}
We let $d=1$ in condition $i)$. Note that $\|\tilde x_i\|\le 1$, thus $\|\tilde x_i^\prime\|\le \|\tilde x_i\|+1\le2$. Since $|\tilde y_i|\le 1$, $\|\tilde\theta\|\le \rho$, and $\mu\ge 0$, one can easily show that
\begin{equation}
\|\frac{\partial C(\tilde \theta)}{\partial \tilde x_i}\mid_{\tilde x_i=\tilde x_i^\prime}\|\le\frac{1}{\lambda}(4\rho+1)(\rho+\|\theta^\prime\|).
\end{equation}
Define $g(\tilde x_i,b)=\frac{1}{\lambda}(4\rho+1)(\rho+\|\theta^\prime\|)\psi e^{-\|b\|}$. Note that $\int g(\tilde x_i,b)db=\frac{1}{\lambda}(4\rho+1)(\rho+\|\theta^\prime\|)<\infty$, thus both condition $i)$ and $ii)$ are satisfied and we can exchange the order of differentiation and integration.

Next we consider $h(\tilde y_i, b)$. According to~\eqref{obj-ridge-sgdy},
\begin{equation}
\frac{\partial C(\tilde\theta)}{\partial \tilde y_i}=\tilde x_i^\top\left(\tilde X^\top \tilde X+(\lambda+\mu) I\right)^{-1}(\tilde\theta -\theta^\prime).
\end{equation}

Since $\|\tilde x_i\|\le 1$, $\|\tilde\theta\|\le \rho$ and $\mu\ge 0$, one can easily show that for all $\tilde y_i^\prime$,
\begin{equation}
\|\frac{\partial C(\tilde \theta)}{\partial \tilde y_i}\mid_{\tilde y_i=\tilde y_i^\prime}\|\le\frac{1}{\lambda}(\rho+\|\theta^\prime\|).
\end{equation}
Define $g(\tilde y_i,b)=\frac{1}{\lambda}(\rho+\|\theta^\prime\|)\psi e^{-\|b\|}$. Note that $\int g(\tilde y_i,b)db=\frac{1}{\lambda}(\rho+\|\theta^\prime\|)<\infty$, thus both condition $i)$ and $ii)$ are satisfied and we can exchange the order of differentiation and integration.

\item Attacking output-perturbed logistic/ridge regression.

The proof for output-perturbed logistic regression and ridge regression are similar to the objective-perturbed learners, thus we omit them here.

\end{enumerate}

For other attack goals such as label-aversion attack and label-targeting attack, the regularity condition also holds and the proof is similar.

	
	
\end{document}